\documentclass[letterpaper, 10 pt, conference]{ieeeconf}  

\IEEEoverridecommandlockouts                              

\usepackage{times}
\usepackage{xcolor}
\usepackage{amsfonts}
\usepackage{gensymb}
\usepackage{caption}
\usepackage{subcaption}

\usepackage{amsmath}

\usepackage{wrapfig}

\usepackage{multicol}

\usepackage{amssymb}
\usepackage{booktabs}
\usepackage{multirow}
\usepackage{graphicx}
\usepackage{paralist}
\usepackage[font=scriptsize,labelfont=sl,labelsep=period]{caption}
\usepackage{subfiles}
\usepackage[export]{adjustbox}

\newtheorem{proposition}{Proposition}
\newtheorem{assumption}{Assmuption}

\newcommand{\SE}{\mathrm{SE}}

\newcommand{\pick}{\mathrm{pick}}
\newcommand{\place}{\mathrm{place}}
\newcommand{\pre}{\mathrm{preplace}}

\usepackage{stfloats}

\usepackage{hyperref}
\hypersetup{hidelinks}
\usepackage{float}
\usepackage{setspace}
\newcommand{\ours}{\textsc{Match Policy}}
\usepackage{array} 

\title{\LARGE \bf
$\ours$: A Simple Pipeline from Point Cloud Registration to Manipulation Policies
}

\author{Haojie Huang$^1$ \qquad Haotian Liu$^2$ \qquad Dian Wang$^1$ \qquad Robin Walters$^{1\ast}$ \qquad Robert Platt$^{1\ast}$\\
$^1$Khoury College of Computer Science, Northeastern University \quad $^2$Worcester Polytechnic Institute\\
{\texttt \{huang.haoj; r.walters; r.platt\} @northeastern.edu \qquad $^{\ast}$Equal Advising}
}

\begin{document}

\maketitle
\thispagestyle{empty}
\pagestyle{empty}

\begin{abstract}
Many manipulation tasks require the robot to rearrange objects relative to one another. Such tasks can be described as a sequence of relative poses between parts of a set of rigid bodies. In this work, we propose $\ours$, a simple but novel pipeline for solving high-precision pick and place tasks. Instead of predicting actions directly, our method registers the pick and place targets to the stored demonstrations. This transfers action inference into a point cloud registration task and enables us to realize nontrivial manipulation policies without any training. $\ours$ is designed to solve high-precision tasks with a key-frame setting. By leveraging the geometric interaction and the symmetries of the task, it achieves extremely high sample efficiency and generalizability to unseen configurations. We demonstrate its state-of-the-art performance across various tasks on
RLbench benchmark compared with several strong baselines and test it on a real robot with six tasks. Videos and code are available on \href{https://haojhuang.github.io/match_page}{\textcolor{orange}{https://haojhuang.github.io/match\_page/}}.

\end{abstract}

\section{introduction}

Many complex manipulation tasks can be decomposed as a sequence of pick-place actions, each of which can further be interpreted as inferring two geometric relationships: the pick pose is the relative pose between the gripper and the pick target, and the place pose is the relative pose between the pick target and the place target. 
Previous imitation learning methods~\cite{shridhar2023perceiver,goyal2023rvt,ke20243d} directly predicted the pick-place actions given the entire observation signal after being trained with a copious amount of demonstrations. However, these methods did not highlight the importance of local geometric relationships and thus struggled to learn high-precision manipulation policies such as those required to solve the \textit{Plug-Charger} and \textit{Insert-Knife} tasks in RLBench~\cite{james2020rlbench}. Meanwhile, recent works~\cite{simeonov2022neural,huang2024imagination, simeonov2023shelving, pan2023tax} leverages segmented point clouds to reason about the geometric interaction between object instances. However, they often require a significant amount of effort before being applied to the real robots. Methods like NDFs~\cite{simeonov2022neural} and its variation~\cite{simeonov2023se} require significant per-object pretraining and thus cannot be simply used on different object sets. Methods like Tax-Pose~\cite{pan2023tax} and RPDiff~\cite{simeonov2023shelving} can only predict a single-step single-task action after hours of training, which dramatically limits their potential on multi-step and long-horizon tasks.

To address the constraints of current methods and provide a convenient tool for robotic pick-place policies that require minimal effort to deploy across different tasks,
we propose $\ours$, a simple pipeline that transfers manipulation policy learning into point cloud registration (PCR). 
$\ours$ constructs a combined point cloud of the desired scene using segmented point clouds, where objects are arranged in the expected configuration. As illustrated in Fig~\ref{fig:macth_policy}, we store a collection of combined point clouds from the demonstration data. During inference, the point clouds of the pick and place objects are registered to these stored point clouds, and the resulting registration poses are used to compute the action.
Unlike the prior works that require heavy training, we realize this pipeline with optimization-based method: $\ours$ takes use of the RANSAC and ICP and produces the pick-place policy immediately after the demonstration collection. 


Our proposed method has a couple of key advantages. First, the PCR step corresponds the local geometric details shown in the demonstration to the new observation, enabling the agent to solve high-precision tasks like \textit{Plug-Charger} and \textit{Insert-knife}.
Second, $\ours$ illustrates great sample efficiency, i.e., the ability to learn good policies with relatively few expert demonstrations. We demonstrate it can achieve the compelling performance with only one demonstration and can generalize to many different novel poses with various experiments. 
Finally, $\ours$ shows high adaptability
when tested with different camera settings, e.g., single camera view and low-resolution cameras, as well as on tasks with long horizons and articulated objects.

Our contribution in this work are as follows. 1) We provide a simple yet novel pipeline that realizes manipulation pick-place policy without any training. 2) We show the benefits of precision and sample efficiency of this method. 3) We demonstrate that it achieves compelling performance on both simulated and real-robot experiments.


\begin{figure*}[t]
    \centering
    \includegraphics[width = 0.85\textwidth]{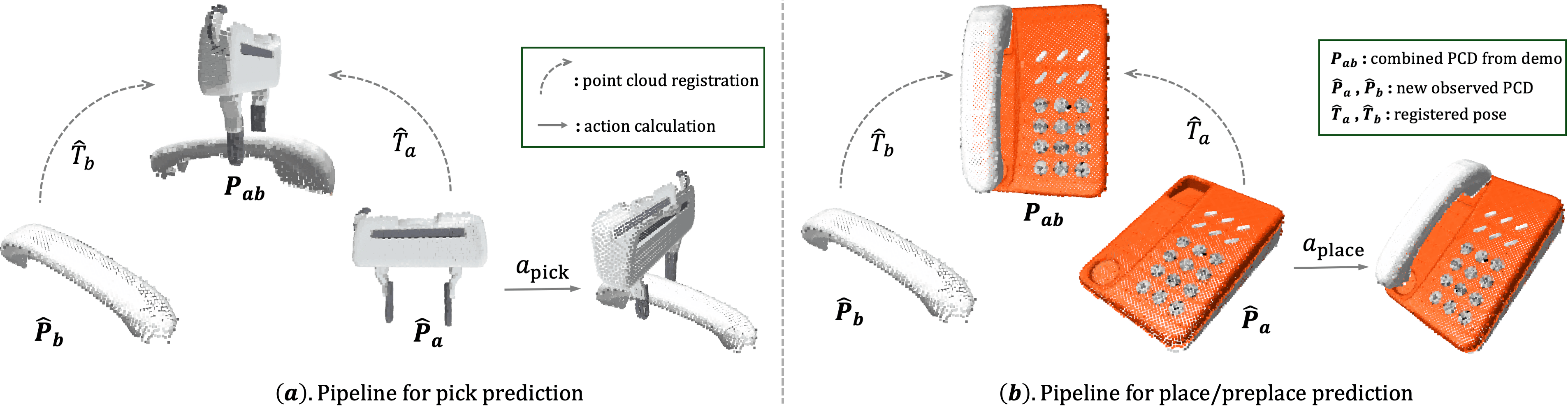}
    \caption{Pipeline of $\ours$. (a). To generate the pick action, we register the gripper ($\hat{P}_a$) and the phone ($\hat{P}_b$) to the demonstrated combined point clouds ($P_{ab}$). The two registration poses $(\hat{T}_a, \hat{T}_b)$ are used to calculate the action transforming the gripper to desired pick configuration. (b). The place action
    prediction follows the similar pipeline.}
    \label{fig:macth_policy}
    \vspace{-0.5cm}
\end{figure*}


\section{related work}

\textbf{Point Cloud Registration.} Point cloud registration (PCR) is defined to find the best transformation that matches two sets of point clouds~\cite{cheng2018registration}. Current methods can be distinguished into non-learning optimization methods and deep learning-based techniques~\cite{zhang2020deep}. The non-learning methods include two representative ones, Iterative Closest Point (ICP)~\cite{besl1992method} and RANSAC~\cite{fischler1981random} as well as their variations. ICP)~\cite{besl1992method} and its variations~\cite{zhang2021fast, yang2015go, rusinkiewicz2019symmetric, zhang2021fast, park2017colored} usually require an initial guess. They search for the closest correspondence points and estimate transformation until convergence. RANSAC-based methods~\cite{le2019sdrsac, fischler1981random, chum2003locally, choi2015robust} can be interpreted as an outlier detection method and have also demonstrated effective registration results.
These non-learning methods are plug-and-play for any objects, although they often require sufficient overlap to guarantee successful registration~\cite{huang2021comprehensive}. 
The current research also focuses on deep learning models to learn the local and global geometric representation to calculate the correspondence~\cite{zhang2024comprehensive}. 
Deep Closet Point (DCP) utilizes DGCNN~\cite{wang2019dynamic} to embed local features and a pointer network to calculate the correspondence~\cite{wang2019deep}. The PRNet~\cite{wang2019prnet} introduced keypoint recognition to address
partial to partial point cloud registration. 
Recently, Methods like Predator~\cite{huang2021predator} and PEAL~\cite{yu2023peal} introduced attention blocks to locate the overlap region and produce correspondence. 
In this work, we investigate non-learning optimization methods on robotic pick-place data.

\textbf{Manipulation Learning on Point Cloud.} 
As a flexible and informative representation of the robot's operating environment, point cloud has demonstrated superior effectiveness compared to other visual formats, such as RGB-D images~\cite{ze20243d,zhu2024point,peri2024point}. Recent works have broadly applied this rich representation across various robotic manipulation problems, including reinforcement learning~\cite{liu2022frame,qin2023dexpoint,xie2023part}, closed-loop policy learning~\cite{ze20243d,wang2024rise,xian2023chaineddiffuser,ma2024hierarchical}, key-point based methods~\cite{gervet2023act3d,chen2023polarnet,ke20243d}, and robotic pick-place~\cite{simeonov2022neural,simeonov2023shelving,ryu2022equivariant,ryu2023diffusion,pan2023tax,eisner2024deep}. However, a major challenge in utilizing previous policy learning frameworks is their computational complexity and the significant effort required to adapt them to new tasks.
In contrast, our method provides a simple and convenient solution to realizing manipulation pick-place policy without parameterization and training. It can be deployed effectively immediately after demonstration collection.

\textbf{Manipulation Learning with Sample Efficiency.} Robotic tasks defined in 3D Euclidean space are invariant to translations, rotations, and reflections which redefine the coordinate frame but do not otherwise alter the task. Recent advancements in equivariant modeling~\cite{e2cnn,deng2021vector,cesa2022a,liao2022equiformer,he2021efficient} provide a powerful tool to encode symmetries in robotics and other related fields~\cite{he2022neural,li2024affine,li2025affine}.
\cite{zhu2022grasp, zhu2023robot,huang2023edge, huorbitgrasp} used equivariant models to leverage pick symmetries for grasp learning. \cite{yang2024equivact} proposed an equivariant policy for deformable and articulated object manipulation on top of pre-trained equivariant visual representation. \cite{zhao2022integrating, zhao2024mathrm} used equivariant maps to enable efficient planning.
Other works~\cite{wang2021q,simeonov2022neural, simeonov2023se, Huang-RSS-22, huang2024leveraging,huang2024fourier,ryu2022equivariant,ryu2023diffusion,pan2023tax,eisner2024deep} 
leverage symmetries in pick and place and achieve high sample efficiency. \cite{jia2024open,huanglanguage} explore the symmetry under the language-conditioned policy and realize few-shot learning with language steerable kernels. Recently, \cite{wang2022so2equivariant, jia2023seil, wang2024equivariant, wang2022onrobot,liu2023continual,kohler2023symmetric,nguyen2023equivariant,nguyen2024symmetry,yang2024equivact,yang2024equibot} realized equivariant close-loop policies and demonstrated better generalization performance with fewer demonstrations. Compared with previous work, our method takes the advantage of point cloud registration to realize the equivariant policy and shows an improvement in sample efficiency. 


\section{Problem Statement}
Considering a set of expert demonstrations $\{\mathcal{D}_i\}_{i=1}^{n}$,  each demonstration $\mathcal{D}_i$ consists of a sequence of pick and place. We represent each pick or place sample with object-centric point clouds and their transformations of the form $(P_a, P_b, T_a, T_b, \ell )$, where $P_a\in \mathbb{R}^{n\times3}$ and $P_b\in \mathbb{R}^{m\times3}$ are point clouds that represent two objects of interest, $T_a \in \mathbb{R}^{4\times 4}$ and $T_b \in \mathbb{R}^{4\times 4}$ are two rigid transformations in $\SE(3)$ represented in homogeneous coordinates that can transform $P_a$ and $P_b$ into the desired configuration,  $\ell$ is the language description explaining the action and objects. In our settings, if $\ell$ illustrates a pick action, $(P_a, P_b)$ will represent the gripper and the pick target.
If it indicates the preplace/place action, $(P_a, P_b)$ indicates the placement and the object to arrange, respectively.
Our goal is to model the policy function $f\colon ({P}_{a}, {P}_{b}, \ell) \mapsto a$ which outputs the gripper movement $a\in \SE(3)$ and can generalize to new observed point clouds in different configurations. The policy is formulated to generate the multi-step pick-place actions in the open-loop manner and each single-step action is parameterized with $(a_{\pick}, a_{\pre}, a_{\place})$\footnote{ It can also be referred to as key-frame action. The preplace action is important to solve complex tasks, e.g., \textit{Inserting} and \textit{Hanging}, without any predefined prior actions.}.

\section{Method}
We first explain the procedure (Fig~\ref{fig:macth_policy}) of $\ours$ which takes the segmented point clouds as input and outputs the key-frame actions $(a_{\pick}, a_{\pre}, a_{\place})$. We denote $\hat{P}_a$ and  $\hat{P}_b$ as the observed point clouds during inference, to distinguish them from the demonstrated point clouds.

\subsection{Procedure of $\ours$}

\textbf{Storing Combined Point Clouds ${P_{ab}}$.} We first construct the combined point cloud $P_{ab}$ from the demonstration sample $(P_a, P_b, T_a, T_b, \ell )$ by 
\begin{equation}
    P_{ab} = T_a \cdot P_a \cup T_b \cdot P_b
\end{equation}
where $\cdot$ transforms the two segmented point clouds with $T_a$ and $T_b$ to the desired configuration, and $\cup$ concatenates the two transformed point clouds. In other words, $P_{ab}$ represents either the desired pick configuration or the desired preplace/place configuration, as shown in Fig~\ref{fig:macth_policy}. Compared to using the entire scene's point cloud, this approach reduces occlusion and filters out irrelevant information.
Each $P_{ab}$ is described by the language description $\ell$.
Taking the \textit{Phone-on-Base} task shown in Fig~\ref{fig:macth_policy} and \ref{fig:3d_task_des} for example, there are three $P_{ab}$ denoted by three descriptions, ``pick up the phone", ``preplace the phone above the base" and ``place the phone on the base". We store each pair of $(\ell, P_{ab})$ as a key-value element for every demonstration. It results in a dictionary across all the tasks and all the demonstrations.


\textbf{Registering $\hat{P}_a$ and $\hat{P}_a$ to $P_{ab}$.} During inference, we first extract the point clouds of objects of interest from observation. 
After retrieving the $P_{ab}$ with the language description $\ell$ as the key, our registration model $f_{r}\colon (\hat{P}_a, \hat{P}_b, P_{ab}) \mapsto (\hat{T}_a, \hat{T}_b)$ outputs the poses that match $\hat{P}_a$ and $\hat{P}_b$ to the combined point cloud $P_{ab}$. We realize $f_r$ with optimization-based registration methods. 
Specifically, we begin by applying RANSAC~\cite{choi2015robust} to obtain the initial alignment, followed by colored ICP~\cite{park2017colored} for iterative refinement.

Apart from inferring the registration pose, we calculate the fitness score $S= \frac{\text{\# of inlier correspondences}} {\text{\# of points in source}}$ to measure the registration quality. We run the registration model $f_r$ with every sample that matches the key for several times with random seeds and calculate the fitness score. We run it multiple times because RANSAC is a stochastic algorithm.
It results in a set of registration results $\{(\hat{T}_a, \hat{T}_b, S_a, S_b)\}$ and we select the best pair of registration poses using the highest average fitness score over $S_a$ and $S_b$.

\textbf{Calculating $a_{\pick}$, $a_{\pre}$ and $a_{\place}$.} After estimating the registration poses $(\hat{T}_a, \hat{T}_b)$ for pick, preplace and place with the language key $\ell$ respectively, we calculate the pick action as the relative pose to arrange the gripper to the current pick target $\hat{P}_b$, i.e., $a_{\pick}=(\hat{T}_b)^{-1}\hat{T}_a$. 
The preplace and place action are determined by moving the pick target $\hat{P}_b$ while keeping the placement $\hat{P}_a$ stationary, to match desired configuration, i.e., $a_{\place}=(\hat{T}_a)^{-1}\hat{T}_b$. Finally, our method outputs $(a_{\pick}, a_{\pre}, a_{\place})$ that can be used to control the robot arm. This process can be repeated to infer a sequence of keyframe actions to solve complex tasks.

\subsection{Sample efficiency Analysis on $\ours$}
\label{theory}
 We then analyze the equivariant property of our method through the lens of equivariance with a mild assumption.
Since RANSAC is stochastic voting scheme, by computing a greater number of iterations, the probability of an optimal registration being produced is increased, especially when the overlapping area is above 50\%. We can assume our registration model $f_r$ is \emph{optimal} after enough running times:
\begin{assumption}
$f_{r}\colon (\hat{P}_a, \hat{P}_b, P_{ab}) \mapsto (\hat{T}_a, \hat{T}_b)$ is \emph{optimal}. For all $g\in\SE(3)$, $f_r$ will have the following properties:
\begin{enumerate}[(a)]
\item $f_r(\hat{P}_a, \hat{P}_b, g\cdot P_{ab})=(g\cdot\hat{T}_a, g\cdot\hat{T}_b)$
\item $f_r(g\cdot \hat{P}_a, \hat{P}_b, P_{ab})=(\hat{T}_a\cdot g^{-1}, \hat{T}_b)$
\item $f_r(\hat{P}_a, g\cdot \hat{P}_b, P_{ab})=(\hat{T}_a, \hat{T}_b \cdot g^{-1})$
\end{enumerate}
\label{assum:fr}
\end{assumption}

With Assumption 1, we conclude three equivariant properties of $\ours$ that improve the sample efficiency. In the following part, we will only include $a_{\pick}$ and $a_{\place}$ to reduce redundancy. We use $f_\pick$ and $f_\place$ to represent the pick and place predictors of the policy function $f$.

\textbf{Invairant Symmetry.} We first show that $\ours$ generates invariant prediction of $(a_{\pick}, a_{\place})$ when the demonstration point clouds $P_{ab}$ transforms.
\begin{proposition}
    $a_{\pick}$ and $a_{\place}$ are invariant to transformation $g\in \SE(3)$ acting on $P_{ab}$.
\end{proposition}
\begin{proof}
  By Assumption~\ref{assum:fr}a, if $P_{ab}$ is transformed by $g\in \SE(3)$, the calculated registration poses are transformed to $g\cdot \hat{T}_a$ and $g\cdot \hat{T}_b$. The new pick action can be calculated as $a'_{\pick} = (g\hat{T}_b)^{-1}g\hat{T}_a= \hat{T}_b^{-1}g^{-1}g\hat{T}_a = a_{\pick}$. Similarly, the new place action $a'_{\place}=a_{\place}$.
\end{proof}

Proposition 1 states that many different demonstrations that produce differently transformed $P_{ab}$ result in the same action prediction using one demonstration in the same group. It enables our method to achieve good performance with very few demonstrations.

\textbf{Bi-equivariant Place Symmetry.} As noted in previous work~\cite{huang2024imagination,huang2024fourier, Huang-RSS-22, ryu2022equivariant, ryu2023diffusion}, the relative place actions that rearrange object B to another object A are bi-equivariant. That is, independent transformations of object A with $g_a \in \SE(3)$ and object B with $g_b \in \SE(3)$ result in a change ($a'_{\place}=g_a a_{\place} g_b^{-1}$) to complete the rearrangement at the new configuration. Leveraging bi-equivariant symmetries can generalize the stored place knowledge to different configurations and improve the sample efficiency. 
\begin{proposition} 
 The place action inference of $\ours$ is bi-equivariant:
 $f_{\place}(g_a\cdot P_a, g_b\cdot P_b) = g_a f_{\place}(P_a,P_b) g_b^{-1} $.
\end{proposition}
\begin{proof}
    If $P_a$ and $P_b$ are transformed by $g_a$ and $g_b$ respectively, the calculated registration poses are $\hat{T}_a (g_a)^{-1}$ and $\hat{T}_b (g_b)^{-1}$ (Assumption~\ref{assum:fr}bc). The new place action can be estimated as $a'_{\place} = (\hat{T}_ag_a^{-1})^{-1}\hat{T}_bg_b^{-1} = g_a\hat{T}_a^{-1}\hat{T}_bg_b^{-1} = g_a a_\place g_b^{-1}$, which satisfies the bi-equivariance.
\end{proof}

The bi-equivariant design of our method significantly improves the sample efficiency. It enables our model to evaluate effectively on a lower dimensional space, the equivalence classes of samples under the SE(3) x SE(3) group.

\textbf{Equivariant Pick Symmetry.} Lastly, we show that the pick action inference is also equivariant to transformations on the pick target, i.e., $f_{\pick}(P_a, g_b\cdot P_b) = g_bf_{\pick}(P_a,P_b)$.

\begin{proof}
    By Assumption~\ref{assum:fr}c, if the pick target is transformed by $g_b$, the outputted registration pose is transformed to $\hat{T}_b{g_b}^{-1}$ and the new pick action can be calculated as $a'_{\pick}= (\hat{T}_bg_b^{-1})^{-1}\hat{T}_a = g_b\hat{T}_b^{-1}\hat{T}_a=g_ba_{\pick}$, which realizes the pick equivariance.
\end{proof}

Proposition 3 states that if the pick target transforms, the pick action will transform accordingly. The equivariant pick symmetry enables our method to acquire a better generalization of the stored pick knowledge to many different poses with few demonstrations.

The above analyses provide the theoretical proof why $\ours$ can be sample efficient. Generally, real-robot data is noisy and the transformation of the objects will not result in the exactly same transformation in the observed point clouds due to occlusion and distortion. Nonetheless, the PCR method remains useful for identifying correspondences between key geometric features. We further evaluate our proposed method across different settings and a variety of experiments, providing detailed analyses of the results.

\section{Implementation}
In our implementation, we first collect the robot data of the traditional form $(o_t, a_t)$, where $o_t$ is the observation captured by one or multiple RGB-D cameras and $a_t=(\pick, \pre, \place)$ defined in the world frame. We then segment $P_a$ and $P_b$ using the masks. The masks can either be ground-truth masks from a simulator or computed using current segmentation methods~\cite{ke2024segment, kirillov2023segment}, which are beyond the scope of this work. The gripper point cloud is directly sampled from the gripper mesh file at its canonical pose. 
For the pick, we store $T_a$ as the $\pick$ pose and $T_b$ as an identity matrix. For the place, $T_a$ is an identity matrix, while $T_b$ is calculated as the relative pose between ${\pick}$ and ${\place}$ poses. 
After constructing $P_{ab}$, we downsample it  by voxelizing the input with a
4mm voxel dimension.
Since dictionary lookups have an average time complexity of $\mathcal{O}(1)$, we store all the combined point clouds from various tasks in a single dictionary.
We also store the color information for each point cloud. 
The RANSAC~\cite{choi2015robust} and colored ICP~\cite{park2017colored} are implemented using Open3d~\cite{zhou2018open3d}.

\section{Simulations}

\begin{figure*}[t]
     \centering
     \begin{subfigure}[b]{0.135\textwidth}
         \centering
         \includegraphics[width=0.99\textwidth]{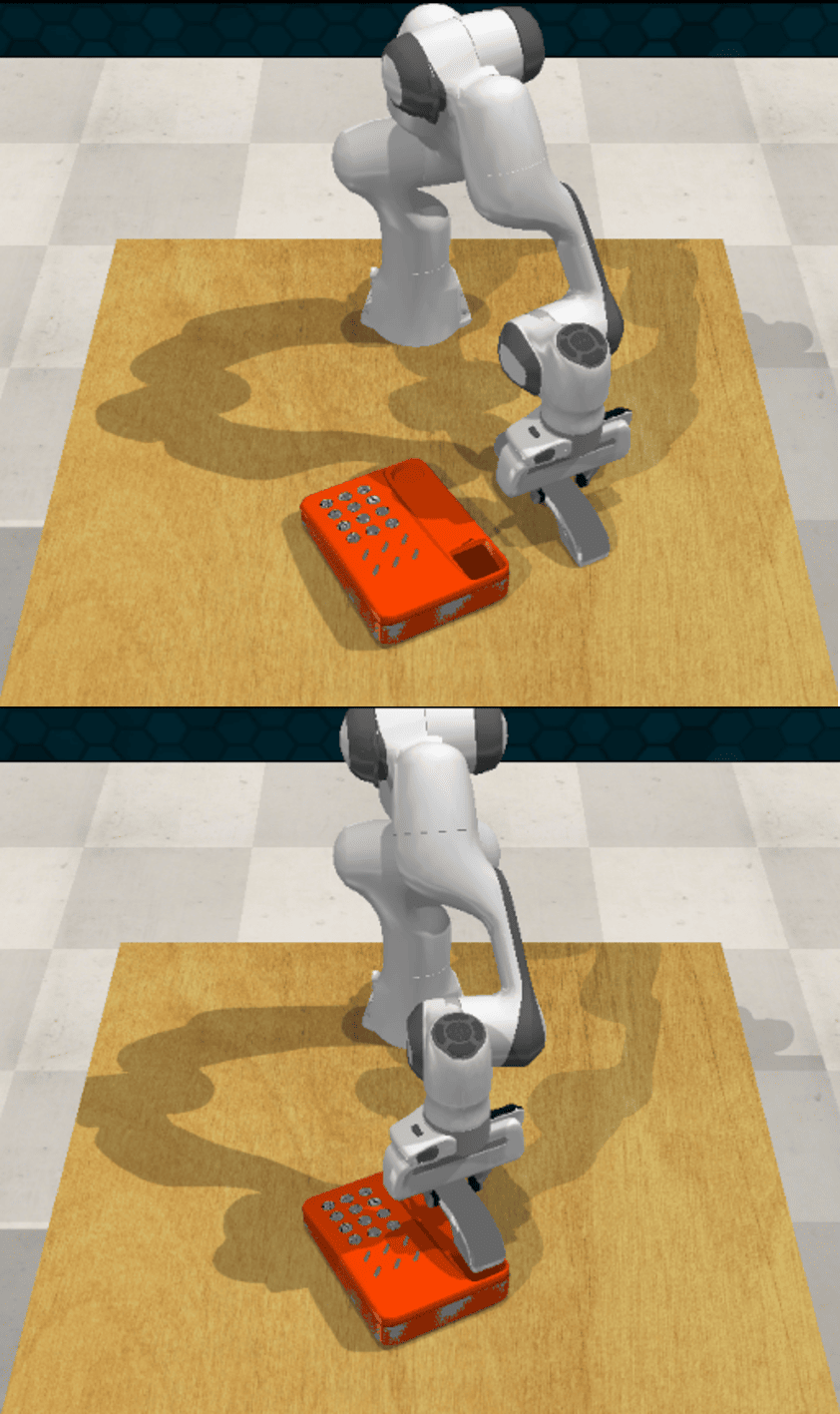}
     \end{subfigure}
     \begin{subfigure}[b]{0.135\textwidth}
         \centering
         \includegraphics[width=0.99\textwidth]{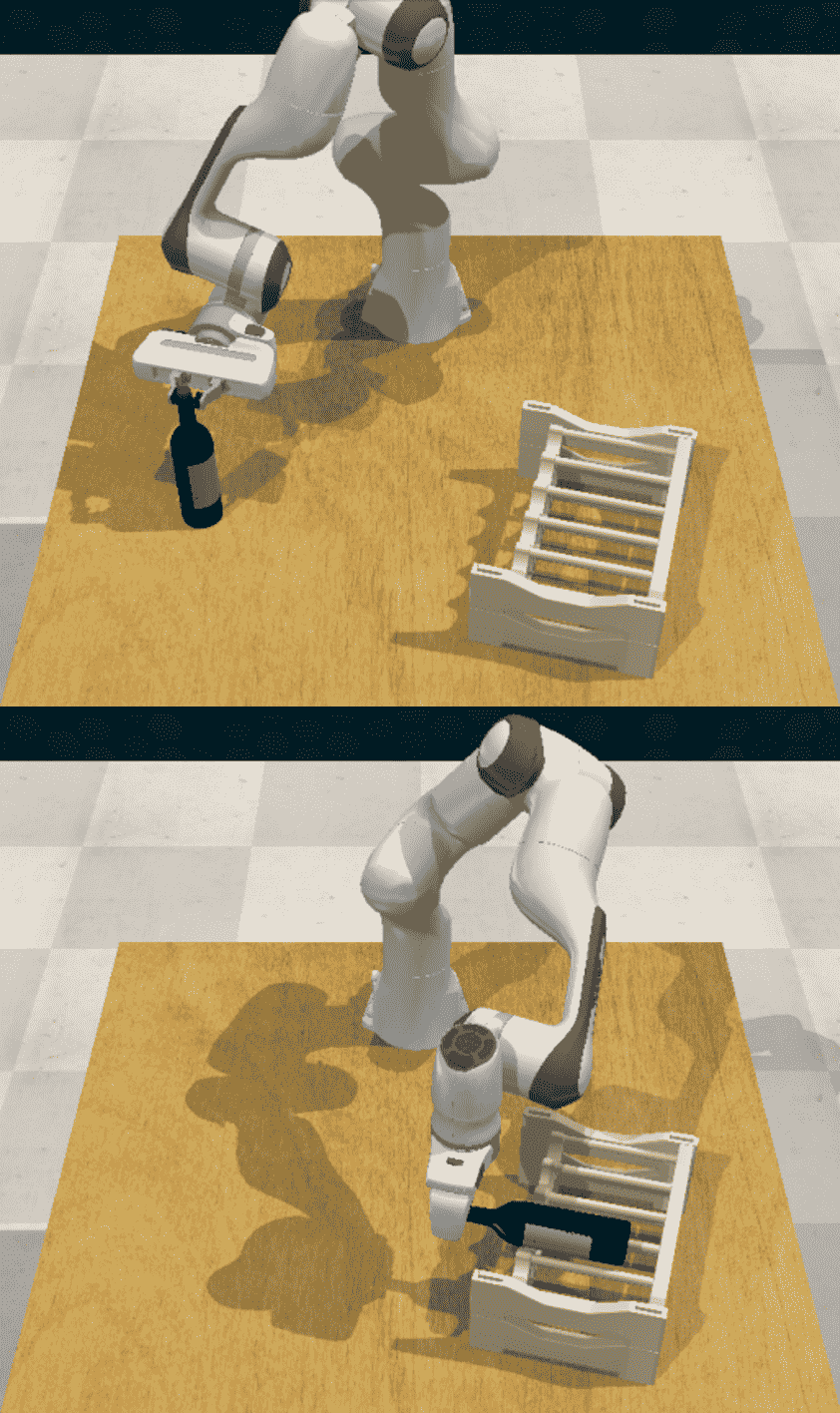}
     \end{subfigure}
     \begin{subfigure}[b]{0.135\textwidth}
         \centering
         \includegraphics[width=0.99\textwidth]{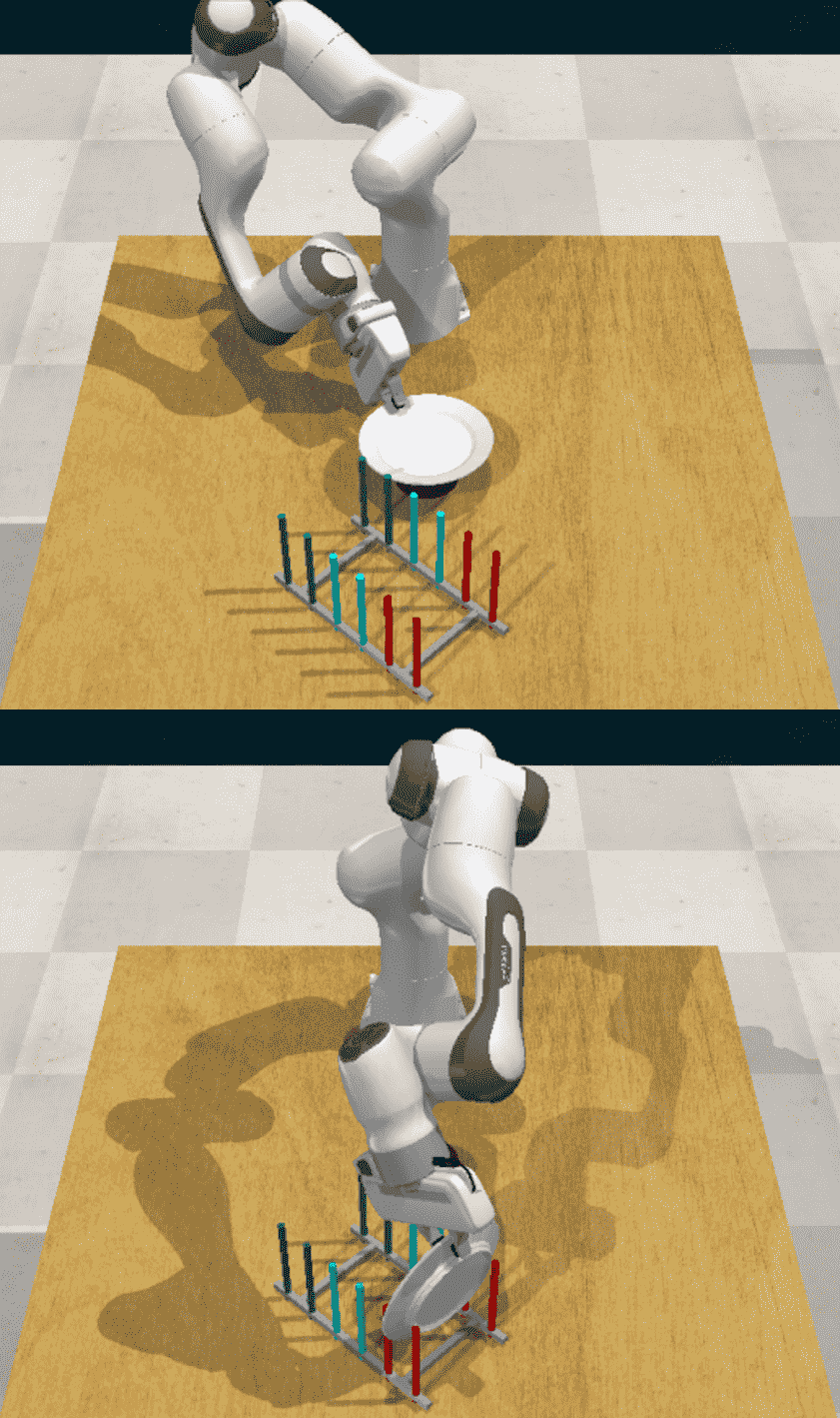}
     \end{subfigure}
     \begin{subfigure}[b]{0.135\textwidth}
         \centering
         \includegraphics[width=0.99\textwidth]{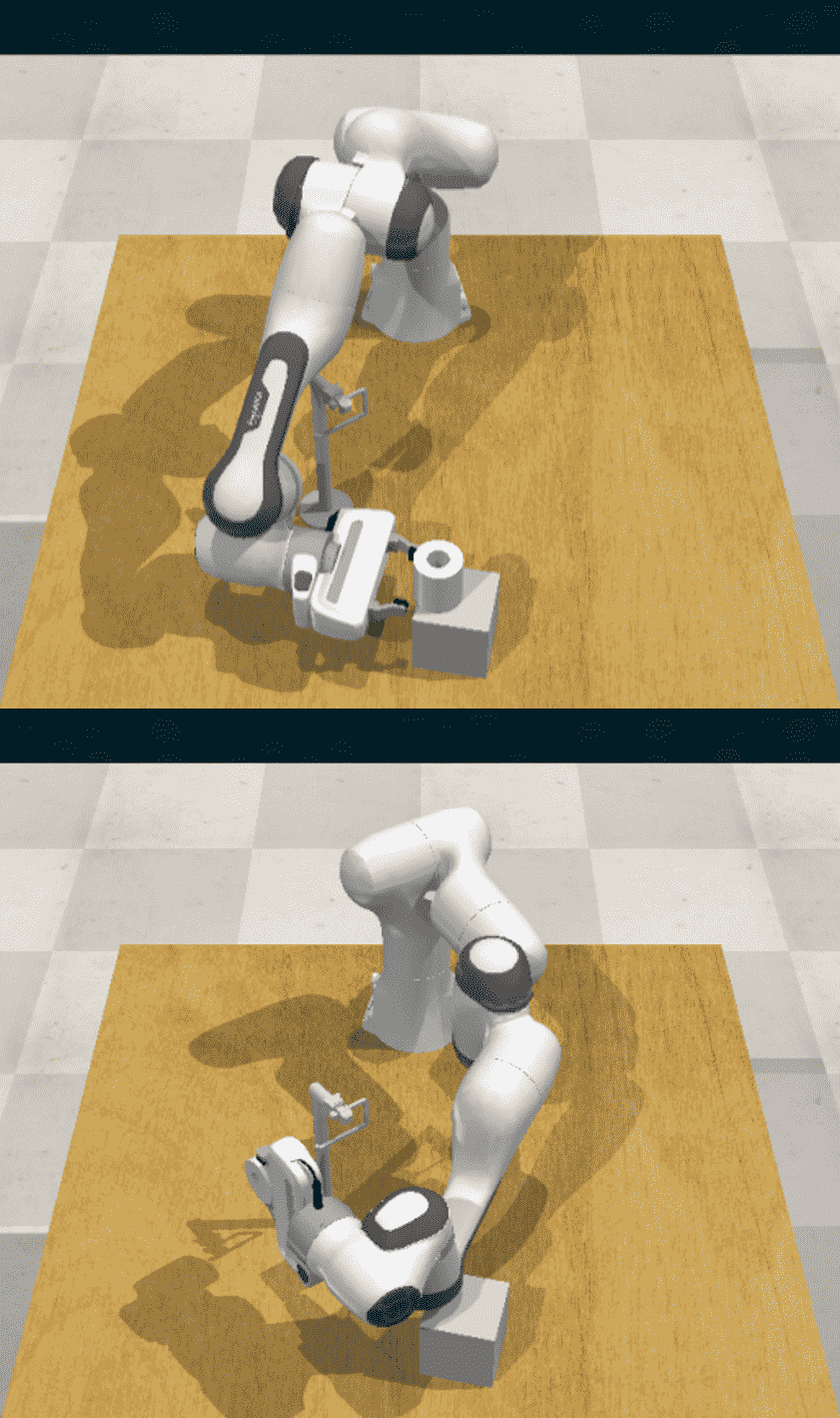}
     \end{subfigure}
      \begin{subfigure}[b]{0.135\textwidth}
         \centering
         \includegraphics[width=0.99\textwidth]{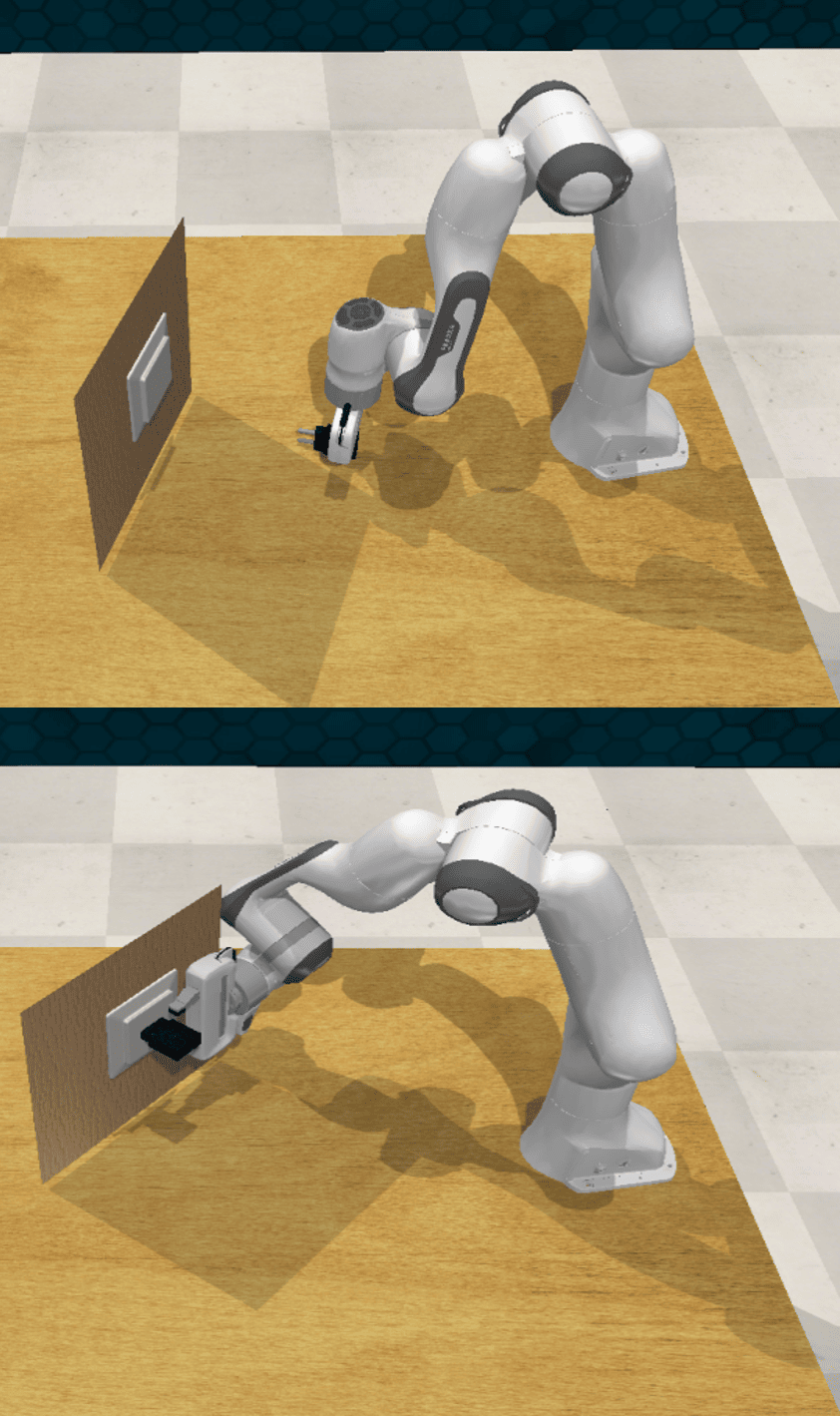}
     \end{subfigure}
      \begin{subfigure}[b]{0.135\textwidth}
         \centering
         \includegraphics[width=0.99\textwidth]{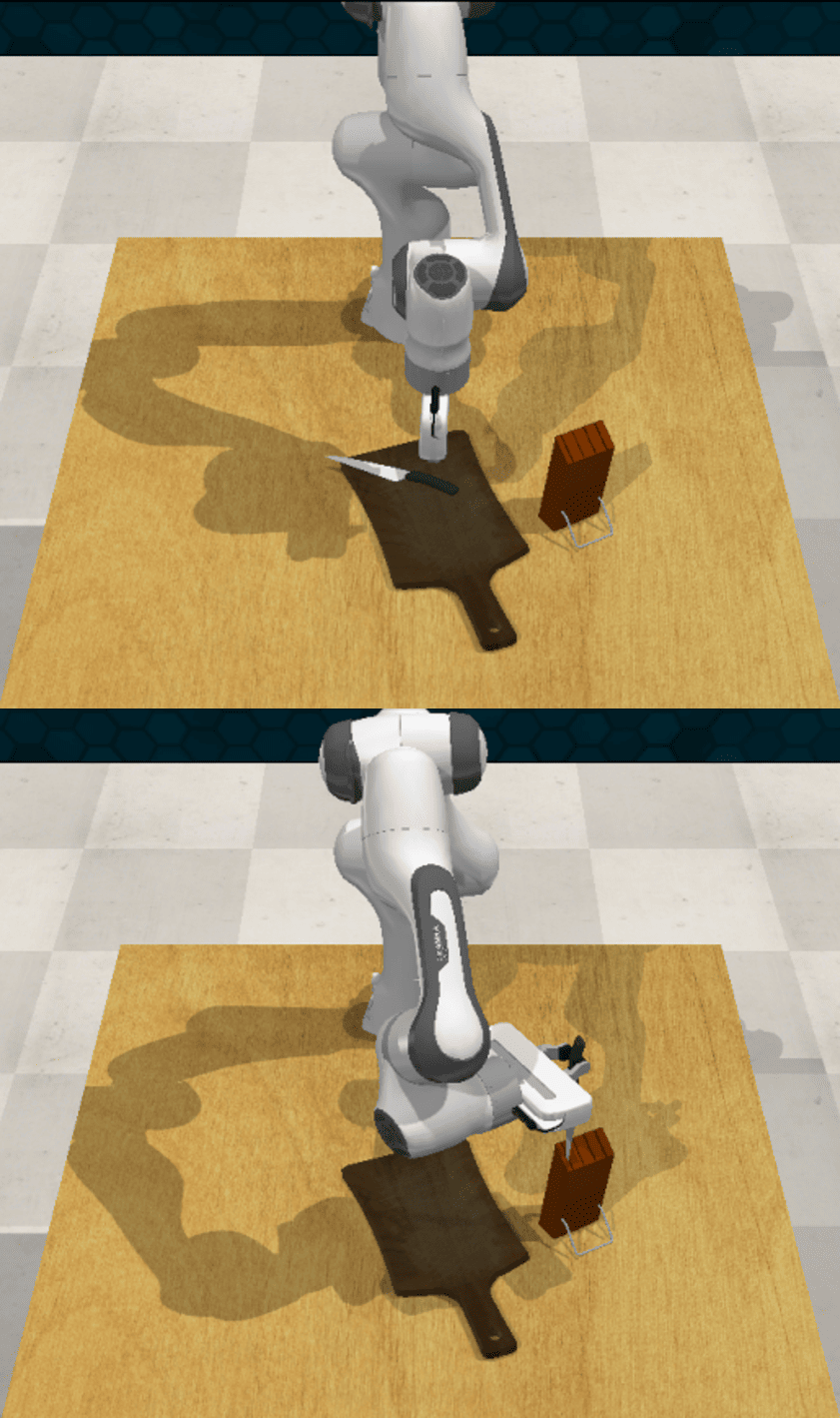}
     \end{subfigure}
     \caption{3D pick-place tasks from RLBench~\cite{james2020rlbench}. The top row shows the initial scene and the bottom row illustrate the completion state. The tasks are: \textit{Phone-on-Base, Stack-Wine, Put-Plate, Slide-Roll, Plug-Charger}, and \textit{Insert-Knife}.\vspace{-0.2cm}}
     \label{fig:3d_task_des}
\end{figure*}

\begin{table*}[t]
    \centering
    \setlength{\tabcolsep}{4pt}
    \fontsize{8pt}{6pt}\selectfont
    \begin{tabular}{cccccccc}
    \toprule
Model & \# demos             & phone-on-base & stack-wine & put-plate & slide-roll & plug-charger & insert-knife 
\\\midrule
\ours{} (ours)  &   1 & 93.33 ($\pm$ 4.62) & {74.67} ($\pm$ 2.31) & 10.67 ($\pm$ 2.31) & 0 & 34.67 ($\pm$ 16.65) & 44.00 ($\pm$ 8.00)  \\
{Imagination Policy~\cite{huang2024imagination}} & 1 & 4.00 ($\pm$ 4.52) & 2.67 ($\pm$ 2.61) & 1.33 ($\pm$ 2.61) & 2.78 ($\pm$ 2.72) & 0 & 0     \\
\midrule

\ours{} (ours)  &   10  & \textbf{100} ($\pm$ 0.00)& \textbf{98.67} ($\pm$ 2.31) & 13.33 ($\pm$ 6.11)& 7.24 ($\pm$ 9.05) & \textbf{40.00} ($\pm$ 4.00) & \textbf{61.33} ($\pm$ 2.31)    \\
Imagination Policy~\cite{huang2024imagination} &  10 &  {90.67} ($\pm$ 2.61) & {97.33} ($\pm$ 2.61) &  {34.67} ($\pm$ 10.45) & \textbf{23.61} ($\pm$ 5.44) & {26.67} ($\pm$ 13.82) & {42.67} ($\pm$ 9.42)               \\

{RPDiff}~\cite{simeonov2023shelving} & {10} & {62.67} ($\pm$ 5.22) &{32.00} ($\pm$ 4.52) & {5.33} ($\pm$ 5.22) & {0} & {0} & {2.67} ($\pm$ 2.61)\\

RVT~\cite{goyal2023rvt}   & 10 & 56.00 ($\pm$ 4.52) & 18.67 ($\pm$ 2.61) & \textbf{53.33} ($\pm$ 6.91) & 0 & 0 & 8.00 ($\pm$ 4.52)                \\
PerAct~\cite{shridhar2023perceiver}   & 10 & 66.67 ($\pm$ 11.39) & 5.33 ($\pm$ 2.62) & 12.00 ($\pm$ 4.52) & 0 & 0 & 0             \\
3D Diffusor Actor~\cite{ke20243d} & 10 & 29.33 ($\pm$ 5.22) & 26.67 ($\pm$ 14.55) & 12.00 ($\pm$ 0.00) & 0 & 0 & 0 \\
\midrule
Key-Frame Expert  &  & 100 & 100 & 74.6 & 56 & 72 & 90.6          \\\bottomrule
\end{tabular}

    \caption{Performance comparisons on RL benchmark. Success rate (\%) on 25 tests when using 1 or 10 demonstration episodes. Results are averaged over 3 runs.}
    \label{tab:3d_results}
    \vspace{-0.7cm}
\end{table*}

We first test our proposed method on several compelling simulated tasks with a limited number of demonstrations. Our primary baselines is Imagination Policy~\cite{huang2024imagination} which is also a bi-equivariant model using segmentations. 
We adopt their simulated experimental settings and baselines to ensure a fair comparison

\subsection{Simulated Experiments}
\textbf{Task Description.} We use the same six RLbench tasks~\cite{james2020rlbench} as used by Imagination Policy~\cite{huang2024imagination}. 
\textit{{Phone-on-Base:}} The agent is asked to pick up the phone and place it onto the phone base correctly.
\textit{{Stack-Wine}}: This task includes grabbing the wine bottle and putting it on the wooden rack at one of three specified locations.
\textit{{Put-Plate}}: The agent must pick up the plate and insert it between the red spokes in the colored dish rack. The colors of other spokes are randomly generated from the full set of 19 color instances.
\textit{{Slide-Roll}}: This task involves grasping the toilet roll and sliding it onto its stand. This task is a high-precision task.
\textit{{Plug-Charger:}} The agent is asked to pick up the charger and plug it into the power supply on the wall. This also requires high precision.
\textit{{Insert-Knife:}} This task requires picking up the knife from the chopping board and sliding it into its slot in the knife block. The different 3D tasks are shown graphically in Fig~\ref{fig:3d_task_des}. During the test, object poses are randomly sampled at the beginning of each episode and the agent must generalize to novel poses.

\textbf{\textbf{Baselines.}} Our method is compared against five strong baselines:
\textit{{Imagination Policy}}~\cite{huang2024imagination} is our primary baseline. It consumes the segmented $P_a$ and $P_b$ to generate the combined point cloud with a point flow model~\cite{wu2023fast} and uses SVD to calculate the registration poses.
\textit{{RPDiff}}~\cite{simeonov2023shelving} consumes segmented $P_a$ and $P_b$ and denoises the relative pose iteratively.
\textit{{PerAct}}
~\cite{shridhar2023perceiver} is a multi-task behavior cloning agent using transformer to process the voxel grids to learn a language-conditioned policy.
\textit{{RVT}}~\cite{goyal2023rvt} projects the 3D observation onto five orthographic images and uses the dense feature map of each image to generate 3D actions.
\textit{{3D Diffuser Actor}}~\cite{ke20243d} is a variation of Diffusion Policy~\cite{chi2023diffusion} that denoises noisy actions conditioned on point cloud features. 
\textit{{Key-Frame Expert}:} Since some tasks are very complex, we also report the performance of the expert agent that uses the key-frame action extracted from the demonstration to measure the effects of path planning.
Other methods like NDFs~\cite{simeonov2022neural} and its variation~\cite{simeonov2023se} are not included since they require per-object pretraining.

\textbf{Settings and metrics.} 
There are four cameras (front, right shoulder, left shoulder, hand) pointing toward the workspace. We use the ground truth mask to segment $P_a$ and $P_b$ for RPDiff~\cite{simeonov2023shelving}, Imagination Policy~\cite{huang2024imagination} and our method.
Since the environments we test on are relatively uncluttered, making it relatively easy to extract object masks from images or voxel maps, segmentation is thus not a performance bottleneck for RVT~\cite{goyal2023rvt} which uses orthographic images or PerAct~\cite{shridhar2023perceiver} which uses voxel maps. 
All methods are evaluated on 25 unseen configurations and each evaluation is averaged over 3 evaluation seeds.

\textbf{Results.} We report the success rate of each method in Table~\ref{tab:3d_results} and draw several findings from it. (1). With 10 demos, $\ours$ outperforms all the baselines on 4 out of 6 tasks. It also performs better with only 1 demo than all baselines trained with 10 demos on 3 out of 6 tasks. (2). $\ours$ achieves better results on \textit{plug-charger} and \textit{insert-knife} which require high precision. (3). The performance gap of $\ours$ between 1 demo and 10 demos is much smaller than that of Imagination-Policy. It validates the invariant property shown in Section~\ref{theory}. (4). $\ours$ shows room for improvement on the \textit{Put-Plate} and \textit{Put-Roll} tasks, likely due to the symmetric nature of the plate and roll, which leads to many unreachable pick-place poses. (1) (2) and (3) demonstrate the sample efficiency and the compelling performance of our proposed method.

\subsection{Test on Different Camera Settings}

\begin{table}[h!]
    \centering
    \setlength{\tabcolsep}{3pt}
    \fontsize{8pt}{5pt}\selectfont
    \begin{tabular}{ccccc}
    \toprule
Model & \# demos  & phone-on-base & stack-wine & insert-knife 
\\\midrule
 $480 \times 480$ images & 10 & {100} & {98.67} & {61.33}     \\ 
\midrule
$128\times 128$ images & 10 & 100 & 92.00 & 40.00 \\ 
\midrule
Single Front View & 10 & 76.00 & 88.00 & 12.00\\
\bottomrule
\end{tabular}
    \caption{Ablation Study on camera settings. Success rate (\%) on 25 tests.\vspace{-0.2cm}}
    \label{tab:camera_results}
    \vspace{-0.2cm}
\end{table}

The ability to adapt to various camera settings is a crucial aspect of a model's robustness and versatility. We test our method with three different camera settings: 1). four RGB-D cameras with high resolution; 2). four RGB-D cameras with low resolution; 3). a single front camera view with high resolution. Table~\ref{tab:camera_results} includes the results of three tasks tested with our method. It shows that low-resolution images slightly decrease the performance and multi-view images provide a more complete observation.

\subsection{Task with Articulated Object}

Unlike rigid bodies, articulated objects consist of several movable parts linked together. We test our proposed method on $\textit{Open Microwave}$ to illustrate its potential for articulated object manipulation. As shown in Fig~\ref{fig:open_microwave}, the task requires the robot to grasp the microwave handle and open the door. We segment the two moveable parts of the microwave and predict the relative pose between the gripper and the door and the relative pose between the door and the frame. The results are reported in Table~\ref{tab:microwave_results}. Please note complex articulated object manipulation can also be implemented in the same way as inferring the relative poses between links. The majority of failure cases occurred when the registration model inaccurately registered the door, leading to errors of place action.

\begin{figure}[t]
    \centering
    \includegraphics[width = 0.4\textwidth]{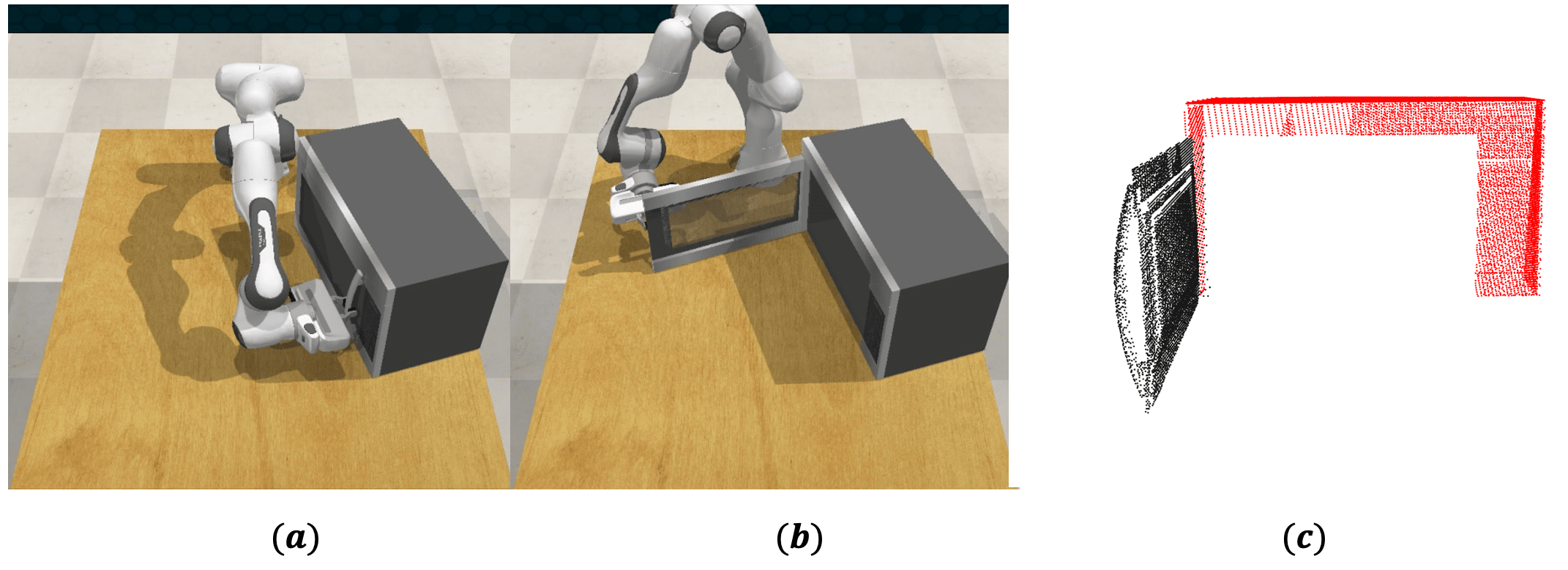}
    \caption{Articulated Object Manipulation: Open Microwave. (a). grasp the handle of the microwave, (b). open the door of the microwave, (c). segmentations of the door with handle (black color) and microwave frame (red color).\vspace{-0.1cm}}
    \label{fig:open_microwave}
    \vspace{-0.2cm}
\end{figure}

\begin{table}[t]
    \centering
    \setlength{\tabcolsep}{8pt}
    \fontsize{8pt}{6pt}\selectfont
    \renewcommand{\arraystretch}{1.5}
    \begin{tabular}{c|c|c|c}
    \toprule
Task  & \# demos & $\ours$ & expert 
\\
\hline
open-microwave & 10 & {32.00} & {92.00}     \\ 
\bottomrule
\end{tabular}
    \caption{Performance on Open-Microwave. Success rate (\%) on 25 tests using 10 demonstrations.\vspace{-0.4cm}}
    \label{tab:microwave_results}
    \vspace{-0.25cm}
\end{table}

\subsection{Task with Long Horizon}

\begin{figure}[H]
    \centering
    \includegraphics[width = 0.45\textwidth]{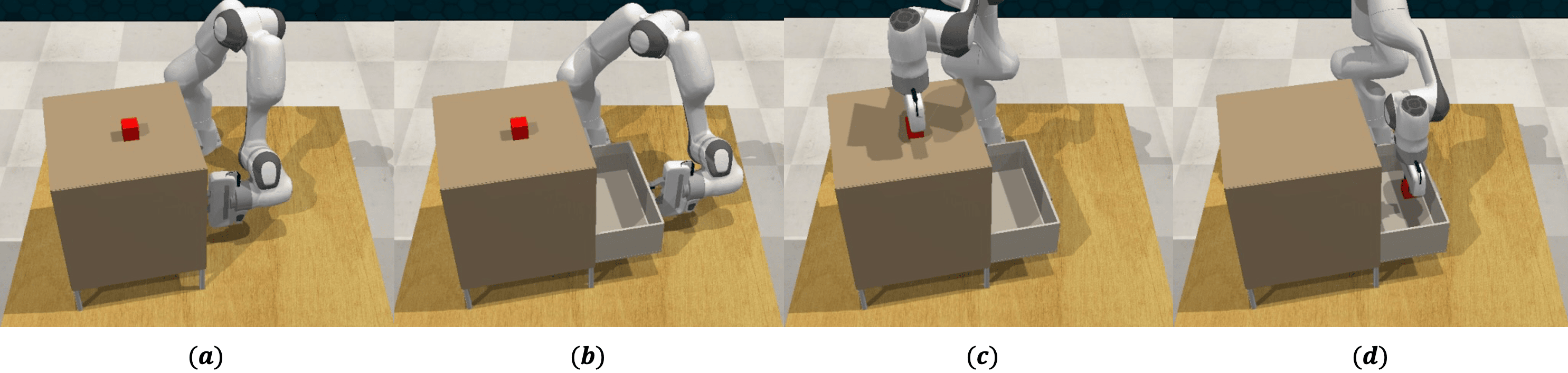}
    \caption{Long Horizon Task: Put Item in Drawer. From left to right: (a). grasp the handle of the drawer; (b). open the drawer; (c). pick up the red block; (d). put the block in the drawer. }
    \label{fig:put_item_in_drawer}
    \vspace{-0.5cm}
\end{figure}

\begin{table}[h!]
    \centering
    \setlength{\tabcolsep}{8pt}
    \fontsize{8pt}{6pt}\selectfont
    \renewcommand{\arraystretch}{1.5}
    \begin{tabular}{c|c|c|c}
    \toprule
Task  & \# demos & $\ours$ & expert 
\\
\hline
 put-item-in-drawer & 10 & {96.00} & {96.00}     \\ 
\bottomrule
\end{tabular}
    \caption{Performance on Put-Item-in-Drawer. Success rate (\%) on 25 tests using 10 demonstrations.\vspace{-0.4cm}}
    \label{tab:put_item_in_drawer_results}
    \vspace{-0.1cm}
\end{table}

Many complex tasks can be decomposed as a sequence of pick-place actions. We test $\ours$ on a long-horizon task
- \textit{Put Item in Drawer}. As shown in Fig~\ref{fig:put_item_in_drawer}, the agent needs to first open the bottom drawer and then pick up the red block and finally put it in the drawer. We address it by inferring two pick-place actions the results are included in Table~\ref{tab:put_item_in_drawer_results}. It shows that our proposed method achieves 96\% success rate which is the same as the oracle agent's performance. 

\section{Real-robot Experiments}
Our proposed method can efficiently be deployed on real-robot tasks without any training, and we test it on 6 real-robot tasks by collecting only 10 demos for each task. To evaluate $\ours$ comprehensively, we assess its performance across different numbers of camera views and fitness scores. 
\begin{figure*}[t]
    \centering
    \includegraphics[width = 0.9\textwidth]{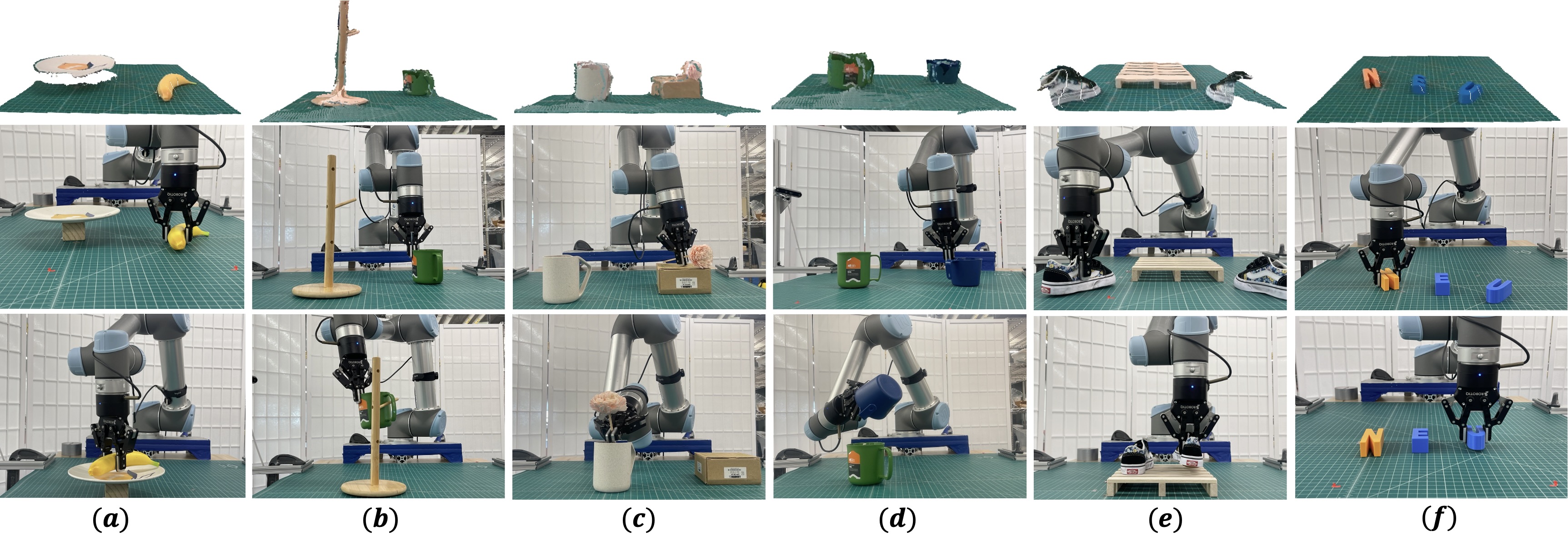}
    \caption{Real-robot tasks. The top row shows the observed real-sensor point clouds, the second row indicates the first pick action and the last row shows the complete state. Task from left to right: (a). putting banana; (b). hanging mug; (c). inserting flower; (d). pouring ball; (e). packing shoes; (f) arranging letters.\vspace{-0.1cm}}
    \label{fig:real_tasks}
    \vspace{-0.1cm}
\end{figure*}

\vspace{-0.2cm}
\textbf{Hardware Settings.}
The experiment is performed on a UR5 robot on a $48\mathrm{cm}\times 48\mathrm{cm}\times48\mathrm{cm}$ workspace. There are three RealSense 455 cameras mounted around the workspace. To collect the demos, we released the UR5 brakes to push the arm physically.
For execution on a robot, it requires a collision-free pick-and-place trajectory that connects the predicted actions.
We use MoveIt with RRT-star as our motion planner to generate the trajectory.

\textbf{Tasks.}
As shown in Fig~\ref{fig:real_tasks}, we test six different real-robot tasks. The action space of all the tasks are defined in $\SE(3)$. {\textit{Putting-Banana:}} This task asks the agent to pick up the banana and put it on the plate. {\textit{Hanging-Mug:}} The robot needs to pick up the mug and hanging it on the poke of the mug holder. {\textit{Inserting-Flower:}} This task includes picking up the flower and plugging it into the base. {\textit{Pouring-Ball:}} The agent is asked to grasp the small blue cup and pour the white ball into the big green cup. {\textit{Packing-Shoes:}} It is a multi-step task and includes picking a pair of shoes and placing them on the shelf. {\textit{Arranging-Letters:}} It is a multi-step task requiring the agent to arrange three letters with two pick-place actions.

\textbf{Results.}
Objects are randomly placed in the workspace during testing, and we run 15 tests for each task. Our results are included in Table~\ref{tab:real_results}. We report the action success rate for each step as well as the task completion rates. For \textit{Putting-Banana}, \textit{Packing-Shoes}, and \textit{Arranging-Letters}, we use a single camera view and set the fitness threshold as 0.7 to filter out low-confidence actions. For the other three tasks, we use three camera views without the fitness threshold. As shown in the first row of Fig~\ref{fig:real_tasks}, real-sensor data is typically noisy due to occlusion and distortion. However, $\ours$ can achieve above 90\% success rate on \textit{Putting-Banana} and \textit{Packing-Shoes} with only 10 demos. We observe a performance drop in \textit{Hanging-Mug} and \textit{Inserting-Flower} and we hypothesize that sensor noise is a key factor of this decline, which causes critical features like the hanger's hook and the mug's handle to disappear. This results in inaccurate registration poses and affects task performance. We provide the real-robot videos in the supplementary materials.


\begin{table*}[t!]
    \centering
    \setlength{\tabcolsep}{4pt}
    \fontsize{8pt}{5pt}\selectfont
    \renewcommand{\arraystretch}{1.5}
    \begin{tabular}{@{}l*{17}{>{\centering\arraybackslash}p{9mm}@{}}}
    \toprule
Task & \multicolumn{2}{c}{putting-banana} & \multicolumn{2}{c}{hanging-mug} & \multicolumn{2}{c}{inserting-flower} & \multicolumn{2}{c}{pouring-ball} & \multicolumn{4}{c}{packing-shoes} & \multicolumn{4}{c}{arranging-letters} 
\\ 
\midrule
\# camera view & \multicolumn{2}{c}{single camera} & \multicolumn{2}{c}{three cameras} & \multicolumn{2}{c}{three cameras} & \multicolumn{2}{c}{three cameras} & \multicolumn{4}{c}{single camera} & \multicolumn{4}{c}{single camera} \\

\midrule

fitness threshold & \multicolumn{2}{c}{0.7} & \multicolumn{2}{c}{0.0} & \multicolumn{2}{c}{0.0} & \multicolumn{2}{c}{0.0} & \multicolumn{4}{c}{0.7} & \multicolumn{4}{c}{0.7} \\

\cmidrule(lr){2-3}  \cmidrule(lr){4-5}  \cmidrule(lr){6-7} \cmidrule(lr){8-9} \cmidrule(lr){10-13} \cmidrule(lr){14-17}
& pick & place & pick & place & pick & place & pick & place & pick & place & pick & place & pick & place & pick & place\\
\midrule
\#success/\#trival & 14/15 & 14/14 & 13/15 & 6/13  & 6/15 & 4/6 & 14/15 & 10/14  & 15/15 & 15/15 & 15/15 & 15/15 & 14/15 & 11/14 & 10/15 & 9/10\\
action success rate (\%) & 93.3 & 100 & 87.0 & 46.2 & 40.0 & 66.7 & 93.3 & 71.4 & 100 & 100 & 100 & 100 & 93.3 & 78.6 & 66.7 & 90 \\
\midrule
completion rate (\%) & \multicolumn{2}{c}{93.3} & \multicolumn{2}{c}{40.0} & \multicolumn{2}{c}{27.0} & \multicolumn{2}{c}{66.7} & \multicolumn{4}{c}{100} & \multicolumn{4}{c}{66.7}  \\
  
\bottomrule
\end{tabular}
    \caption{Performance on real-robot experiments. Success rate (\%) on 15 tests for each task using 10 demonstrations.\vspace{-0.3cm}}
    \label{tab:real_results}
    \vspace{-0.2cm}
\end{table*}

\section{Conclusion}
In this work, we proposed $\ours$, a simple yet effective pipeline that leverages point cloud registration to address manipulation pick-place tasks. Our goal is to provide a convenient and practical tool for pick-place policies that requires minimal effort to deploy across different tasks. $\ours$ demonstrates significant improvement in sample efficiency and strong performance across various tasks. It can also be applied to the long-horizon task and articulated object manipulation. We also provide a theoretical analysis of its performance, focusing on symmetric properties. Finally, we validate our proposed method on a real robot through six tasks and different settings.

One limitation of this work is that it can only output an open-loop policy instead of a high-frequency closed-loop policy. 
We believe the pick-place policy holds significant value for solving real-world challenges, particularly in industrial settings such as picking, packing, sorting, and stacking.
Another limitation is that the current formulation requires segmenting the object of interest. Fortunately, a large number of SOTA segmentation methods~\cite{kirillov2023segment, ke2024segment} provide a convenient tool to address it. Lastly, this work assumes the object is seen and cannot be generalized to novel objects. We believe the current research on unseen object registration~\cite{qi2025twobytwo,chen2022neural, zhang2024generative, zhang2024comprehensive} can gradually address this issue and we leave it as our feature work.

\section{Acknowledgments}
{This project were supported in part by NSF 1750649, NSF 2107256, NSF 2314182, NSF 2134178, NSF 2409351, and NASA 80NSSC19K1474.}

\clearpage
\newpage

\newpage
\bibliographystyle{IEEEtran}
\bibliography{references}

\begin{thebibliography}{10}
\providecommand{\url}[1]{#1}
\csname url@rmstyle\endcsname
\providecommand{\newblock}{\relax}
\providecommand{\bibinfo}[2]{#2}
\providecommand\BIBentrySTDinterwordspacing{\spaceskip=0pt\relax}
\providecommand\BIBentryALTinterwordstretchfactor{4}
\providecommand\BIBentryALTinterwordspacing{\spaceskip=\fontdimen2\font plus
\BIBentryALTinterwordstretchfactor\fontdimen3\font minus \fontdimen4\font\relax}
\providecommand\BIBforeignlanguage[2]{{%
\expandafter\ifx\csname l@#1\endcsname\relax
\typeout{** WARNING: IEEEtran.bst: No hyphenation pattern has been}%
\typeout{** loaded for the language `#1'. Using the pattern for}%
\typeout{** the default language instead.}%
\else
\language=\csname l@#1\endcsname
\fi
#2}}

\bibitem{shridhar2023perceiver}
M.~Shridhar, L.~Manuelli, and D.~Fox, ``Perceiver-actor: A multi-task transformer for robotic manipulation,'' in \emph{Conference on Robot Learning}.\hskip 1em plus 0.5em minus 0.4em\relax PMLR, 2023, pp. 785--799.

\bibitem{goyal2023rvt}
A.~Goyal, J.~Xu, Y.~Guo, V.~Blukis, Y.-W. Chao, and D.~Fox, ``Rvt: Robotic view transformer for 3d object manipulation,'' in \emph{Conference on Robot Learning}.\hskip 1em plus 0.5em minus 0.4em\relax PMLR, 2023, pp. 694--710.

\bibitem{ke20243d}
T.-W. Ke, N.~Gkanatsios, and K.~Fragkiadaki, ``3d diffuser actor: Policy diffusion with 3d scene representations,'' \emph{arXiv preprint arXiv:2402.10885}, 2024.

\bibitem{james2020rlbench}
S.~James, Z.~Ma, D.~R. Arrojo, and A.~J. Davison, ``Rlbench: The robot learning benchmark \& learning environment,'' \emph{IEEE Robotics and Automation Letters}, vol.~5, no.~2, pp. 3019--3026, 2020.

\bibitem{simeonov2022neural}
A.~Simeonov, Y.~Du, A.~Tagliasacchi, J.~B. Tenenbaum, A.~Rodriguez, P.~Agrawal, and V.~Sitzmann, ``Neural descriptor fields: Se (3)-equivariant object representations for manipulation,'' in \emph{2022 International Conference on Robotics and Automation (ICRA)}.\hskip 1em plus 0.5em minus 0.4em\relax IEEE, 2022, pp. 6394--6400.

\bibitem{huang2024imagination}
H.~Huang, K.~Schmeckpeper, D.~Wang, O.~Biza, Y.~Qian, H.~Liu, M.~Jia, R.~Platt, and R.~Walters, ``Imagination policy: Using generative point cloud models for learning manipulation policies,'' \emph{arXiv preprint arXiv:2406.11740}, 2024.

\bibitem{simeonov2023shelving}
A.~Simeonov, A.~Goyal, L.~Manuelli, L.~Yen-Chen, A.~Sarmiento, A.~Rodriguez, P.~Agrawal, and D.~Fox, ``Shelving, stacking, hanging: Relational pose diffusion for multi-modal rearrangement,'' \emph{arXiv preprint arXiv:2307.04751}, 2023.

\bibitem{pan2023tax}
C.~Pan, B.~Okorn, H.~Zhang, B.~Eisner, and D.~Held, ``Tax-pose: Task-specific cross-pose estimation for robot manipulation,'' in \emph{Conference on Robot Learning}.\hskip 1em plus 0.5em minus 0.4em\relax PMLR, 2023, pp. 1783--1792.

\bibitem{simeonov2023se}
A.~Simeonov, Y.~Du, Y.-C. Lin, A.~R. Garcia, L.~P. Kaelbling, T.~Lozano-P{\'e}rez, and P.~Agrawal, ``Se (3)-equivariant relational rearrangement with neural descriptor fields,'' in \emph{Conference on Robot Learning}.\hskip 1em plus 0.5em minus 0.4em\relax PMLR, 2023, pp. 835--846.

\bibitem{cheng2018registration}
L.~Cheng, S.~Chen, X.~Liu, H.~Xu, Y.~Wu, M.~Li, and Y.~Chen, ``Registration of laser scanning point clouds: A review,'' \emph{Sensors}, vol.~18, no.~5, p. 1641, 2018.

\bibitem{zhang2020deep}
Z.~Zhang, Y.~Dai, and J.~Sun, ``Deep learning based point cloud registration: an overview,'' \emph{Virtual Reality \& Intelligent Hardware}, vol.~2, no.~3, pp. 222--246, 2020.

\bibitem{besl1992method}
P.~J. Besl and N.~D. McKay, ``Method for registration of 3-d shapes,'' in \emph{Sensor fusion IV: control paradigms and data structures}, vol. 1611.\hskip 1em plus 0.5em minus 0.4em\relax Spie, 1992, pp. 586--606.

\bibitem{fischler1981random}
M.~A. Fischler and R.~C. Bolles, ``Random sample consensus: a paradigm for model fitting with applications to image analysis and automated cartography,'' \emph{Communications of the ACM}, vol.~24, no.~6, pp. 381--395, 1981.

\bibitem{zhang2021fast}
J.~Zhang, Y.~Yao, and B.~Deng, ``Fast and robust iterative closest point,'' \emph{IEEE Transactions on Pattern Analysis and Machine Intelligence}, vol.~44, no.~7, pp. 3450--3466, 2021.

\bibitem{yang2015go}
J.~Yang, H.~Li, D.~Campbell, and Y.~Jia, ``Go-icp: A globally optimal solution to 3d icp point-set registration,'' \emph{IEEE transactions on pattern analysis and machine intelligence}, vol.~38, no.~11, pp. 2241--2254, 2015.

\bibitem{rusinkiewicz2019symmetric}
S.~Rusinkiewicz, ``A symmetric objective function for icp,'' \emph{ACM Transactions on Graphics (TOG)}, vol.~38, no.~4, pp. 1--7, 2019.

\bibitem{park2017colored}
J.~Park, Q.-Y. Zhou, and V.~Koltun, ``Colored point cloud registration revisited,'' in \emph{Proceedings of the IEEE international conference on computer vision}, 2017, pp. 143--152.

\bibitem{le2019sdrsac}
H.~M. Le, T.-T. Do, T.~Hoang, and N.-M. Cheung, ``Sdrsac: Semidefinite-based randomized approach for robust point cloud registration without correspondences,'' in \emph{Proceedings of the IEEE/CVF conference on computer vision and pattern recognition}, 2019, pp. 124--133.

\bibitem{chum2003locally}
O.~Chum, J.~Matas, and J.~Kittler, ``Locally optimized ransac,'' in \emph{Pattern Recognition: 25th DAGM Symposium, Magdeburg, Germany, September 10-12, 2003. Proceedings 25}.\hskip 1em plus 0.5em minus 0.4em\relax Springer, 2003, pp. 236--243.

\bibitem{choi2015robust}
S.~Choi, Q.-Y. Zhou, and V.~Koltun, ``Robust reconstruction of indoor scenes,'' in \emph{Proceedings of the IEEE conference on computer vision and pattern recognition}, 2015, pp. 5556--5565.

\bibitem{huang2021comprehensive}
X.~Huang, G.~Mei, J.~Zhang, and R.~Abbas, ``A comprehensive survey on point cloud registration,'' \emph{arXiv preprint arXiv:2103.02690}, 2021.

\bibitem{zhang2024comprehensive}
Y.-X. Zhang, J.~Gui, X.~Cong, X.~Gong, and W.~Tao, ``A comprehensive survey and taxonomy on point cloud registration based on deep learning,'' \emph{arXiv preprint arXiv:2404.13830}, 2024.

\bibitem{wang2019dynamic}
Y.~Wang, Y.~Sun, Z.~Liu, S.~E. Sarma, M.~M. Bronstein, and J.~M. Solomon, ``Dynamic graph cnn for learning on point clouds,'' \emph{ACM Transactions on Graphics (tog)}, vol.~38, no.~5, pp. 1--12, 2019.

\bibitem{wang2019deep}
Y.~Wang and J.~M. Solomon, ``Deep closest point: Learning representations for point cloud registration,'' in \emph{Proceedings of the IEEE/CVF international conference on computer vision}, 2019, pp. 3523--3532.

\bibitem{wang2019prnet}
------, ``Prnet: Self-supervised learning for partial-to-partial registration,'' \emph{Advances in neural information processing systems}, vol.~32, 2019.

\bibitem{huang2021predator}
S.~Huang, Z.~Gojcic, M.~Usvyatsov, A.~Wieser, and K.~Schindler, ``Predator: Registration of 3d point clouds with low overlap,'' in \emph{Proceedings of the IEEE/CVF Conference on computer vision and pattern recognition}, 2021, pp. 4267--4276.

\bibitem{yu2023peal}
J.~Yu, L.~Ren, Y.~Zhang, W.~Zhou, L.~Lin, and G.~Dai, ``Peal: Prior-embedded explicit attention learning for low-overlap point cloud registration,'' in \emph{Proceedings of the IEEE/CVF Conference on Computer Vision and Pattern Recognition}, 2023, pp. 17\,702--17\,711.

\bibitem{ze20243d}
Y.~Ze, G.~Zhang, K.~Zhang, C.~Hu, M.~Wang, and H.~Xu, ``3d diffusion policy,'' \emph{arXiv preprint arXiv:2403.03954}, 2024.

\bibitem{zhu2024point}
H.~Zhu, Y.~Wang, D.~Huang, W.~Ye, W.~Ouyang, and T.~He, ``Point cloud matters: Rethinking the impact of different observation spaces on robot learning,'' \emph{arXiv preprint arXiv:2402.02500}, 2024.

\bibitem{peri2024point}
S.~Peri, I.~Lee, C.~Kim, L.~Fuxin, T.~Hermans, and S.~Lee, ``Point cloud models improve visual robustness in robotic learners,'' \emph{arXiv preprint arXiv:2404.18926}, 2024.

\bibitem{liu2022frame}
M.~Liu, X.~Li, Z.~Ling, Y.~Li, and H.~Su, ``Frame mining: a free lunch for learning robotic manipulation from 3d point clouds,'' \emph{arXiv preprint arXiv:2210.07442}, 2022.

\bibitem{qin2023dexpoint}
Y.~Qin, B.~Huang, Z.-H. Yin, H.~Su, and X.~Wang, ``Dexpoint: Generalizable point cloud reinforcement learning for sim-to-real dexterous manipulation,'' in \emph{Conference on Robot Learning}.\hskip 1em plus 0.5em minus 0.4em\relax PMLR, 2023, pp. 594--605.

\bibitem{xie2023part}
P.~Xie, R.~Chen, S.~Chen, Y.~Qin, F.~Xiang, T.~Sun, J.~Xu, G.~Wang, and H.~Su, ``Part-guided 3d rl for sim2real articulated object manipulation,'' \emph{IEEE Robotics and Automation Letters}, 2023.

\bibitem{wang2024rise}
C.~Wang, H.~Fang, H.-S. Fang, and C.~Lu, ``Rise: 3d perception makes real-world robot imitation simple and effective,'' \emph{arXiv preprint arXiv:2404.12281}, 2024.

\bibitem{xian2023chaineddiffuser}
Z.~Xian, N.~Gkanatsios, T.~Gervet, T.-W. Ke, and K.~Fragkiadaki, ``Chaineddiffuser: Unifying trajectory diffusion and keypose prediction for robotic manipulation,'' in \emph{7th Annual Conference on Robot Learning}, 2023.

\bibitem{ma2024hierarchical}
X.~Ma, S.~Patidar, I.~Haughton, and S.~James, ``Hierarchical diffusion policy for kinematics-aware multi-task robotic manipulation,'' in \emph{Proceedings of the IEEE/CVF Conference on Computer Vision and Pattern Recognition}, 2024, pp. 18\,081--18\,090.

\bibitem{gervet2023act3d}
T.~Gervet, Z.~Xian, N.~Gkanatsios, and K.~Fragkiadaki, ``Act3d: 3d feature field transformers for multi-task robotic manipulation,'' in \emph{7th Annual Conference on Robot Learning}, 2023.

\bibitem{chen2023polarnet}
S.~Chen, R.~Garcia, C.~Schmid, and I.~Laptev, ``Polarnet: 3d point clouds for language-guided robotic manipulation,'' \emph{arXiv preprint arXiv:2309.15596}, 2023.

\bibitem{ryu2022equivariant}
H.~Ryu, H.-i. Lee, J.-H. Lee, and J.~Choi, ``Equivariant descriptor fields: Se (3)-equivariant energy-based models for end-to-end visual robotic manipulation learning,'' \emph{arXiv preprint arXiv:2206.08321}, 2022.

\bibitem{ryu2023diffusion}
H.~Ryu, J.~Kim, J.~Chang, H.~S. Ahn, J.~Seo, T.~Kim, J.~Choi, and R.~Horowitz, ``Diffusion-edfs: Bi-equivariant denoising generative modeling on se (3) for visual robotic manipulation,'' \emph{arXiv preprint arXiv:2309.02685}, 2023.

\bibitem{eisner2024deep}
B.~Eisner, Y.~Yang, T.~Davchev, M.~Vecerik, J.~Scholz, and D.~Held, ``Deep se (3)-equivariant geometric reasoning for precise placement tasks,'' \emph{arXiv preprint arXiv:2404.13478}, 2024.

\bibitem{e2cnn}
M.~Weiler and G.~Cesa, ``{General E(2)-Equivariant Steerable CNNs},'' in \emph{Conference on Neural Information Processing Systems (NeurIPS)}, 2019.

\bibitem{deng2021vector}
C.~Deng, O.~Litany, Y.~Duan, A.~Poulenard, A.~Tagliasacchi, and L.~J. Guibas, ``Vector neurons: A general framework for so (3)-equivariant networks,'' in \emph{Proceedings of the IEEE/CVF International Conference on Computer Vision}, 2021, pp. 12\,200--12\,209.

\bibitem{cesa2022a}
\BIBentryALTinterwordspacing
G.~Cesa, L.~Lang, and M.~Weiler, ``A program to build {E(N)}-equivariant steerable {CNN}s,'' in \emph{International Conference on Learning Representations}, 2022. [Online]. Available: \url{https://openreview.net/forum?id=WE4qe9xlnQw}
\BIBentrySTDinterwordspacing

\bibitem{liao2022equiformer}
Y.-L. Liao and T.~Smidt, ``Equiformer: Equivariant graph attention transformer for 3d atomistic graphs,'' \emph{arXiv preprint arXiv:2206.11990}, 2022.

\bibitem{he2021efficient}
L.~He, Y.~Chen, Y.~Dong, Y.~Wang, Z.~Lin, \emph{et~al.}, ``Efficient equivariant network,'' \emph{Advances in Neural Information Processing Systems}, vol.~34, pp. 5290--5302, 2021.

\bibitem{he2022neural}
L.~He, Y.~Chen, Z.~Shen, Y.~Yang, and Z.~Lin, ``Neural epdos: Spatially adaptive equivariant partial differential operator based networks,'' in \emph{The Eleventh International Conference on Learning Representations}, 2022.

\bibitem{li2024affine}
Y.~Li, Y.~Qiu, Y.~Chen, L.~He, and Z.~Lin, ``Affine equivariant networks based on differential invariants,'' in \emph{Proceedings of the IEEE/CVF Conference on Computer Vision and Pattern Recognition}, 2024, pp. 5546--5556.

\bibitem{li2025affine}
\BIBentryALTinterwordspacing
Y.~Li, Y.~Qiu, Y.~Chen, and Z.~Lin, ``Affine steerable equivariant layer for canonicalization of neural networks,'' in \emph{The Thirteenth International Conference on Learning Representations}, 2025. [Online]. Available: \url{https://openreview.net/forum?id=5i6ZZUjCA9}
\BIBentrySTDinterwordspacing

\bibitem{zhu2022grasp}
X.~Zhu, D.~Wang, O.~Biza, G.~Su, R.~Walters, and R.~Platt, ``Sample efficient grasp learning using equivariant models,'' \emph{Proceedings of Robotics: Science and Systems (RSS)}, 2022.

\bibitem{zhu2023robot}
X.~Zhu, D.~Wang, G.~Su, O.~Biza, R.~Walters, and R.~Platt, ``On robot grasp learning using equivariant models,'' \emph{Autonomous Robots}, vol.~47, no.~8, pp. 1175--1193, 2023.

\bibitem{huang2023edge}
H.~Huang, D.~Wang, X.~Zhu, R.~Walters, and R.~Platt, ``Edge grasp network: A graph-based se (3)-invariant approach to grasp detection,'' in \emph{2023 IEEE International Conference on Robotics and Automation (ICRA)}.\hskip 1em plus 0.5em minus 0.4em\relax IEEE, 2023, pp. 3882--3888.

\bibitem{huorbitgrasp}
B.~Hu, X.~Zhu, D.~Wang, Z.~Dong, H.~Huang, C.~Wang, R.~Walters, and R.~Platt, ``Orbitgrasp: {SE} (3)-equivariant grasp learning,'' in \emph{8th Annual Conference on Robot Learning}, 2024.

\bibitem{yang2024equivact}
J.~Yang, C.~Deng, J.~Wu, R.~Antonova, L.~Guibas, and J.~Bohg, ``Equivact: Sim (3)-equivariant visuomotor policies beyond rigid object manipulation,'' in \emph{2024 IEEE International Conference on Robotics and Automation (ICRA)}.\hskip 1em plus 0.5em minus 0.4em\relax IEEE, 2024, pp. 9249--9255.

\bibitem{zhao2022integrating}
L.~Zhao, X.~Zhu, L.~Kong, R.~Walters, and L.~L. Wong, ``Integrating symmetry into differentiable planning with steerable convolutions,'' \emph{arXiv preprint arXiv:2206.03674}, 2022.

\bibitem{zhao2024mathrm}
L.~Zhao, H.~Li, T.~Pad{\i}r, H.~Jiang, and L.~L. Wong, ``E(2) equivariant graph planning for navigation,'' \emph{IEEE Robotics and Automation Letters}, 2024.

\bibitem{wang2021q}
D.~Wang, R.~Walters, X.~Zhu, and R.~Platt, ``{Equivariant {$Q$} Learning in Spatial Action Spaces},'' in \emph{5th Annual Conference on Robot Learning}, 2021.

\bibitem{Huang-RSS-22}
H.~Huang, D.~Wang, R.~Walters, and R.~Platt, ``{Equivariant Transporter Network},'' in \emph{Proceedings of Robotics: Science and Systems}, New York City, NY, USA, June 2022.

\bibitem{huang2024leveraging}
H.~Huang, D.~Wang, A.~Tangri, R.~Walters, and R.~Platt, ``Leveraging symmetries in pick and place,'' \emph{The International Journal of Robotics Research}, p. 02783649231225775, 2024.

\bibitem{huang2024fourier}
\BIBentryALTinterwordspacing
H.~Huang, O.~L. Howell, D.~Wang, X.~Zhu, R.~Platt, and R.~Walters, ``Fourier transporter: Bi-equivariant robotic manipulation in 3d,'' in \emph{The Twelfth International Conference on Learning Representations}, 2024. [Online]. Available: \url{https://openreview.net/forum?id=UulwvAU1W0}
\BIBentrySTDinterwordspacing

\bibitem{jia2024open}
M.~Jia, H.~Huang, Z.~Zhang, C.~Wang, L.~Zhao, D.~Wang, J.~X. Liu, R.~Walters, R.~Platt, and S.~Tellex, ``Open-vocabulary pick and place via patch-level semantic maps,'' \emph{arXiv preprint arXiv:2406.15677}, 2024.

\bibitem{huanglanguage}
H.~Huang, M.~Jia, Z.~Zhang, O.~Biza, L.~Zhao, R.~Walters, and R.~Platt, ``Language conditioned equivariant grasp.''

\bibitem{wang2022so2equivariant}
\BIBentryALTinterwordspacing
D.~Wang, R.~Walters, and R.~Platt, ``{$\mathrm{SO}(2)$}-equivariant reinforcement learning,'' in \emph{International Conference on Learning Representations}, 2022. [Online]. Available: \url{https://openreview.net/forum?id=7F9cOhdvfk_}
\BIBentrySTDinterwordspacing

\bibitem{jia2023seil}
M.~Jia, D.~Wang, G.~Su, D.~Klee, X.~Zhu, R.~Walters, and R.~Platt, ``Seil: Simulation-augmented equivariant imitation learning,'' in \emph{2023 IEEE International Conference on Robotics and Automation (ICRA)}.\hskip 1em plus 0.5em minus 0.4em\relax IEEE, 2023, pp. 1845--1851.

\bibitem{wang2024equivariant}
D.~Wang, S.~Hart, D.~Surovik, T.~Kelestemur, H.~Huang, H.~Zhao, M.~Yeatman, J.~Wang, R.~Walters, and R.~Platt, ``Equivariant diffusion policy,'' \emph{arXiv preprint arXiv:2407.01812}, 2024.

\bibitem{wang2022onrobot}
\BIBentryALTinterwordspacing
D.~Wang, M.~Jia, X.~Zhu, R.~Walters, and R.~Platt, ``On-robot learning with equivariant models,'' in \emph{6th Annual Conference on Robot Learning}, 2022. [Online]. Available: \url{https://openreview.net/forum?id=K8W6ObPZQyh}
\BIBentrySTDinterwordspacing

\bibitem{liu2023continual}
S.~Liu, M.~Xu, P.~Huang, X.~Zhang, Y.~Liu, K.~Oguchi, and D.~Zhao, ``{Continual Vision-based Reinforcement Learning with Group Symmetries},'' in \emph{Conference on Robot Learning}.\hskip 1em plus 0.5em minus 0.4em\relax PMLR, 2023, pp. 222--240.

\bibitem{kohler2023symmetric}
C.~Kohler, A.~S. Srikanth, E.~Arora, and R.~Platt, ``Symmetric models for visual force policy learning,'' \emph{arXiv preprint arXiv:2308.14670}, 2023.

\bibitem{nguyen2023equivariant}
H.~H. Nguyen, A.~Baisero, D.~Klee, D.~Wang, R.~Platt, and C.~Amato, ``Equivariant reinforcement learning under partial observability,'' in \emph{Conference on Robot Learning}.\hskip 1em plus 0.5em minus 0.4em\relax PMLR, 2023, pp. 3309--3320.

\bibitem{nguyen2024symmetry}
H.~Nguyen, T.~Kozuno, C.~C. Beltran-Hernandez, and M.~Hamaya, ``Symmetry-aware reinforcement learning for robotic assembly under partial observability with a soft wrist,'' \emph{arXiv preprint arXiv:2402.18002}, 2024.

\bibitem{yang2024equibot}
J.~Yang, Z.-a. Cao, C.~Deng, R.~Antonova, S.~Song, and J.~Bohg, ``Equibot: Sim (3)-equivariant diffusion policy for generalizable and data efficient learning,'' \emph{arXiv preprint arXiv:2407.01479}, 2024.

\bibitem{ke2024segment}
L.~Ke, M.~Ye, M.~Danelljan, Y.-W. Tai, C.-K. Tang, F.~Yu, \emph{et~al.}, ``Segment anything in high quality,'' \emph{Advances in Neural Information Processing Systems}, vol.~36, 2024.

\bibitem{kirillov2023segment}
A.~Kirillov, E.~Mintun, N.~Ravi, H.~Mao, C.~Rolland, L.~Gustafson, T.~Xiao, S.~Whitehead, A.~C. Berg, W.-Y. Lo, \emph{et~al.}, ``Segment anything,'' in \emph{Proceedings of the IEEE/CVF International Conference on Computer Vision}, 2023, pp. 4015--4026.

\bibitem{zhou2018open3d}
Q.-Y. Zhou, J.~Park, and V.~Koltun, ``Open3d: A modern library for 3d data processing,'' \emph{arXiv preprint arXiv:1801.09847}, 2018.

\bibitem{wu2023fast}
L.~Wu, D.~Wang, C.~Gong, X.~Liu, Y.~Xiong, R.~Ranjan, R.~Krishnamoorthi, V.~Chandra, and Q.~Liu, ``Fast point cloud generation with straight flows,'' in \emph{Proceedings of the IEEE/CVF conference on computer vision and pattern recognition}, 2023, pp. 9445--9454.

\bibitem{chi2023diffusion}
C.~Chi, S.~Feng, Y.~Du, Z.~Xu, E.~Cousineau, B.~Burchfiel, and S.~Song, ``Diffusion policy: Visuomotor policy learning via action diffusion,'' \emph{arXiv preprint arXiv:2303.04137}, 2023.

\bibitem{qi2025twobytwo}
Y.~Qi, Y.~Ju, T.~Wei, C.~Chu, L.~L.~S. Wong, and H.~Xu, ``Two by two: Learning multi-task pairwise objects assembly for generalizable robot manipulation,'' in \emph{Proceedings of the IEEE/CVF Conference on Computer Vision and Pattern Recognition (CVPR)}, 2025.

\bibitem{chen2022neural}
Y.-C. Chen, H.~Li, D.~Turpin, A.~Jacobson, and A.~Garg, ``Neural shape mating: Self-supervised object assembly with adversarial shape priors,'' in \emph{Proceedings of the IEEE/CVF Conference on Computer Vision and Pattern Recognition}, 2022, pp. 12\,724--12\,733.

\bibitem{zhang2024generative}
J.~Zhang, M.~Wu, and H.~Dong, ``Generative category-level object pose estimation via diffusion models,'' \emph{Advances in Neural Information Processing Systems}, vol.~36, 2024.

\end{thebibliography}
\end{document}